\documentclass{article}

\usepackage[round]{natbib}

\usepackage[accepted]{icml2023}

\newcommand{\ourmethod}{\textsc{Loci}\xspace}
\newcommand{\ourmethodmle}{\ensuremath{\textsc{Loci}_{\text{M}}}\xspace}
\newcommand{\ourmethodhsic}{\ensuremath{\textsc{Loci}_{\text{H}}}\xspace}
\newcommand{\ourtitle}{On the Identifiability and Estimation of Causal Location-Scale Noise Models}
\newcommand{\urltocode}{\url{https://github.com/AlexImmer/loci}}

\usepackage{eda} %
\usepackage[utf8]{inputenc} %
\usepackage[T1]{fontenc}    %
\usepackage{url}            %
\usepackage{booktabs}       %
\usepackage{amsfonts}       %
\usepackage{nicefrac}       %
\usepackage{microtype}      %
\usepackage{xcolor}         %
\usepackage{cancel}
\usepackage{bm}
\usepackage{enumerate}
\usepackage{thmtools} %
\newif\ifalternativeversion
\alternativeversiontrue %
\newif\ifnotblind
\notblindtrue %

\usepackage{amsmath,amsfonts,bm}

\def\1{\bm{1}}

\def\vzero{{\bm{0}}}

\def\vtheta{{\bm{\theta}}}
\def\veta{{\bm{\eta}}}
\def\vphi{{\bm{\phi}}}
\def\vpsi{{\bm{\psi}}}
\def\valpha{{\bm{\alpha}}}
\def\vxi{{\bm{\xi}}}

\def\vf{{\bm{f}}}

\def\vn{{\bm{n}}}

\def\vw{{\bm{w}}}
\def\vx{{\bm{x}}}
\def\vy{{\bm{y}}}

\def\mI{{\bm{I}}}

\def\mPhi{{\bm{\Phi}}}
\def\mPsi{{\bm{\Psi}}}

\newcommand{\inv}{^{-1}}
\newcommand{\hadamard}{\circ}
\DeclareMathAlphabet{\mathsfit}{\encodingdefault}{\sfdefault}{m}{sl}
\SetMathAlphabet{\mathsfit}{bold}{\encodingdefault}{\sfdefault}{bx}{n}

\newcommand{\gauss}{\mathcal{N}}

\newcommand{\E}{\mathbb{E}}

\newcommand{\R}{\mathbb{R}}

\newcommand{\diag}{\operatorname{diag}}

\DeclareMathOperator*{\argmax}{arg\,max}
\DeclareMathOperator*{\argmin}{arg\,min}

\definecolor{lightgray}{gray}{0.85}
\definecolor{lightlightgray}{gray}{0.9}
\definecolor{C1}{HTML}{1F77B4}
\definecolor{C2}{HTML}{FF7F0E}
\definecolor{C3}{HTML}{2CA02C}
\definecolor{C4}{HTML}{D62728}
\definecolor{C5}{HTML}{9467BD}
\definecolor{Cg}{HTML}{7F7F7F}
\colorlet{C1light}{C1!90!white}
\colorlet{C2light}{C2!90!white}
\colorlet{C3light}{C3!90!white}
\colorlet{C4light}{C4!90!white}
\colorlet{C5light}{C5!90!white}
\colorlet{Cglight}{Cg!90!white}
\colorlet{C1lighter}{C1!40!white}
\colorlet{C2lighter}{C2!40!white}
\colorlet{C3lighter}{C3!40!white}
\colorlet{C4lighter}{C4!40!white}
\colorlet{C5lighter}{C5!40!white}
\colorlet{C1vlight}{C1!20!white}
\colorlet{C2vlight}{C2!20!white}
\colorlet{C3vlight}{C3!20!white}
\colorlet{C4vlight}{C4!20!white}
\colorlet{C5vlight}{C5!20!white}

\makeatletter
\newcommand{\oset}[3][0ex]{%
  \mathrel{\mathop{#3}\limits^{
    \vbox to#1{\kern-2\ex@
    \hbox{$\scriptstyle#2$}\vss}}}}
\makeatother

\newcommand{\transpose}{^\mathrm{\textsf{\tiny T}}}

\DeclarePairedDelimiterX{\infdivx}[2]{(}{)}{%
  #1\;\delimsize\|\;#2%
}

\DeclareMathAlphabet{\mathbcal}{OMS}{cmsy}{b}{n} %

\newcommand{\fModels}{\ensuremath{\mathcal{F}}\xspace}

\newcommand{\SN}{\ensuremath{T}\xspace}
\newcommand{\Sn}{\ensuremath{i}\xspace}

\newcommand{\Indep}{\mathop{\perp\!\!\!\perp}\nolimits}

\DeclareMathOperator{\I}{\rm{I}}

\newtheorem{theorem}{Theorem}

\newtheorem{assumption}{Assumption}
\newtheorem{definition}{Definition}

\newenvironment{proof}{\paragraph{Proof:}}{\hfill$\square$}

\tikzset{node/.style={black, draw=black, circle, minimum size=0.7cm, scale=0.8}} 
\tikzset{dummy/.style={black, draw=black, circle, minimum size=0.55cm, scale=0.7}} 
\tikzset{latent/.style={black, draw=black, fill=lightgray, circle, minimum size=0.55cm, scale=0.7}} 

\tikzset{causes/.style={->,very thick,  color=black}} 
\tikzset{causesxor/.style={->,very thick, dashed, color=black}} 
\tikzset{causesxoro/.style={o->,very thick, dashed, color=black}} 
\tikzset{connected/.style={o-o,very thick, color=black}} 
\tikzset{connectedd/.style={o-o,very thick, dashed,  color=black}} 
\tikzset{ocauses/.style={o->,very thick,  color=black}} 
\tikzset{confounder/.style={<->,very thick,  color=black}} 
\tikzset{confounderxor/.style={<->,very thick, dashed, color=black}} 
\tikzset{confounderl/.style={<->,very thick,  color=black, bend left=45}} 
\tikzset{confounderr/.style={<->,very thick,  color=black, bend right=45}} 

\let\svthefootnote\thefootnote
\newcommand\blfootnote[1]{%
  \let\thefootnote\relax%
  \footnotetext{#1}%
  \let\thefootnote\svthefootnote%
}

\pgfplotsset{compat=1.18}

\icmltitlerunning{\ourtitle}

\begin{document}

\twocolumn[
\icmltitle{\ourtitle}

\icmlsetsymbol{equal}{*}

\begin{icmlauthorlist}
\icmlauthor{Alexander Immer}{eth-cs,mpi-is}
\icmlauthor{Christoph Schultheiss}{eth-stat}
\icmlauthor{Julia E Vogt}{eth-cs}
\icmlauthor{Bernhard Sch{\"o}lkopf}{eth-cs,mpi-is}\\
\icmlauthor{Peter B{\"u}hlmann}{eth-stat}
\icmlauthor{Alexander Marx}{eth-cs,eth-ai}
\end{icmlauthorlist}

\icmlaffiliation{eth-cs}{Department of Computer Science, ETH Zurich, Switzerland}
\icmlaffiliation{mpi-is}{Max Planck Institute for Intelligent Systems, T\"ubingen, Germany}
\icmlaffiliation{eth-stat}{Seminar for Statistics, ETH Zurich, Switzerland}
\icmlaffiliation{eth-ai}{AI Center, ETH Zurich, Switzerland}

\icmlcorrespondingauthor{Alexander Immer}{alexander.immer@inf.ethz.ch}
\icmlcorrespondingauthor{Alexander Marx}{alexander.marx@inf.ethz.ch}

\icmlkeywords{Causality, Causal Discovery, Location-Scale, Heteroscedasticity, Linear Models}

\vskip 0.3in
]

\printAffiliationsAndNotice{}  %

\begin{abstract}
  We study the class of location-scale or heteroscedastic noise models (LSNMs), in which the effect $Y$ can be written as a function of the cause $X$ and a noise source $N$ independent of $X$, which may be scaled by a positive function $g$ over the cause, i.e., $Y = f(X) + g(X)N$. Despite the generality of the model class, we show the causal direction is identifiable up to some pathological cases. To empirically validate these theoretical findings, we propose two estimators for LSNMs: an estimator based on (non-linear) feature maps, and one based on neural networks. Both model the conditional distribution of $Y$ given $X$ as a Gaussian parameterized by its natural parameters.
  When the feature maps are correctly specified, we prove that our estimator is jointly concave, and a consistent estimator for the cause-effect identification task. Although the the neural network does not inherit those guarantees, it can fit functions of arbitrary complexity, and reaches state-of-the-art performance across benchmarks.
\end{abstract}

\section{Introduction}

Distinguishing cause from effect, given only observational data, is a fundamental problem in many natural sciences such as medicine or biology, and lies at the core of causal discovery. Without any prior knowledge, distinguishing whether $X$ causes $Y$ ($X \to Y$), or whether $Y$ causes $X$ ($Y \to X$) is unattainable~\citep{pearl:00:models}. When assuming a properly restricted structural causal model (SCM), however, the true graph that generated the data is identifiable~\citep{peters:12:ifmoc}. In its most general form, a structural causal model expresses the effect $Y$ as a function of the cause $X$ and an independent noise term $N$, that is, $Y = f(X,N)$.

There exists a vast literature that derived assumptions under which restricted SCMs are identifiable. That is, there exists no backward model which fulfills the same modeling assumptions as the assumed forward model. For simple linear additive noise models (ANMs), for example, it has been shown that if the data are non-Gaussian, then there exists no backward model, such that the cause can be modeled as a linear function of the effect and an additive independent noise term~\citep{shimizu:06:lingam}. Instead, for all such backward models, the noise will depend on the effect. Apart from linear models, identifiability results have been established for non-linear ANMs~\citep{hoyer:09:nonlinear,buhlmann:14:cam},
where $Y = f(X) + N$ and post non-linear (PNL) noise models~\citep{zhang:09:post-nonlinear,zhang:15:loglikelihood}, where $Y = f_2(f_1(X) + N)$. Besides identifiability, it has also been shown that consistent estimators for ANMs exist~\citep{kpotufe:14:consistency-anm}.

\begin{figure}[t!]
\centering
\includegraphics{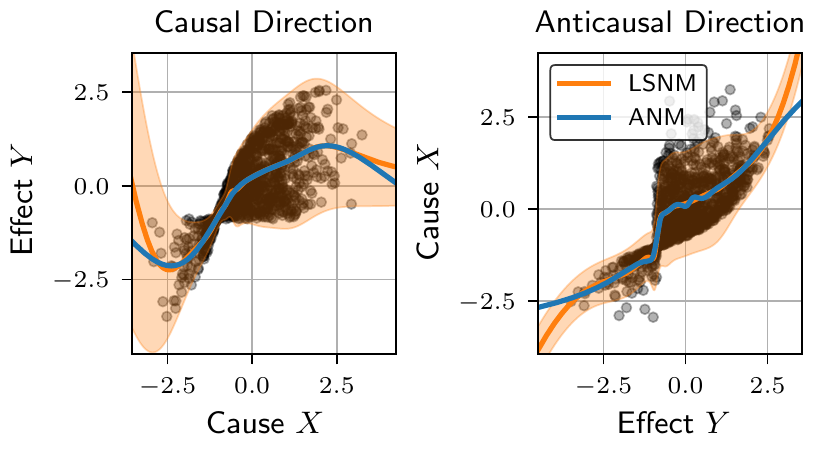}
\vspace{-2.5em}
\caption{Location-scale and additive noise model fits on MNU pair 55. 
Both models show a similar mean regression but the LSNM can model the scale of the variance.
Therefore, the LSNM correctly identifies the causal direction with a log-likelihood difference of $\approx 0.155$ while the ANM identifies the wrong direction with difference $\approx -0.001$.}
\label{fig:lsnm-example}
\vspace{-0.5em}
\end{figure}

Here, we focus on location-scale noise models (LSNMs) or heteroscedastic noise models,\!\footnote{Heteroscedastic noise refers to the setting where the variance of the noise is non-constant and  \emph{depends} on the value of $X$. To emphasize that we model $N$ as an independent source scaled by a function over $X$, we use the notion of location-scale noise.} where the effect $Y$ is expressed as $Y = f(X) + g(X)N$. LSNMs generalize the classical ANM setting by allowing the noise source $N$ to be scaled by a positive function $g(x)$. Naturally, if $g(x) = 1$, an LSNM simplifies to an ANM. In Fig.~\ref{fig:lsnm-example}, we provide an example from a synthetic benchmark~\citep{tagasovska:20:bqcd}, in which the ground truth model follows an LSNM. The ANM-based approach not only identifies the wrong direction, but also, the causal and anti-causal model have near identical log-likelihoods. Thus, the decision is merely a coin flip. The LSNM model, however, identifies the true direction with high confidence. Simply put, we can learn both $f(x)$ and $g(x)$ for the causal direction and can thus separate the independent noise source from $X$, whereas in the anti-causal direction we cannot find such a pair of functions. Hence, the estimated noise remains skewed.

The previous example assumes that we already know the answer to the following questions: Under which assumptions are LSNMs \emph{identifiable}? How and under which assumptions can we \emph{consistently estimate $f(x)$ and $g(x)$}? How can we combine these results to arrive at a score that allows us to \emph{identify cause and effect from observational data}?

\paragraph{Contributions} In this paper, we address each of the questions above. In Sec.~\ref{sec:identifiability}, we formally define LSNMs and show that apart from some pathological cases---which are known for the special case of Gaussian noise~\citep{khemakhem:21:autoregressive}---no backward model exists.\!\footnote{Concurrent work comes to a similar conclusion~\citep{strobl:22:grci}, however, we provide a different proof and clarify the implications of the result further.} 

In Sec.~\ref{sec:consistent-estimation}, we study the estimation of LSNMs under the Gaussian noise assumption, which relates to the problem of maximizing the conditional log-likelihood of $Y$ given $X$ via heteroscedastic regression. Typical estimators aim to learn $f(x)$ as the mean of the Gaussian and $g(x)$ as its variance, which leads to a non-concave likelihood function. We instead propose to relate $f(x)$ and $g(x)$ to the natural parametrization of a Gaussian, with corresponding jointly concave log-likelihood. 
Further, we show that we can consistently estimate LSNMs via (non-linear) feature maps if they are correctly specified. 
To allow for a more general class of functions, we propose a second estimator based on neural networks (NNs). Due to overparametrization, the NN-based estimator does not have guarantees. In practice, however, it reaches state-of-the-art performance, outperforming the variant based on feature maps.

To use our estimators for cause-effect inference, we propose \ourmethod (\textbf{lo}cation-scale \textbf{c}ausal \textbf{i}nference) in Sec.~\ref{sec:ce-inference}. Concretely, we propose a variant of \ourmethod via log-likelihood estimation, extending the approach of \cite{zhang:09:post-nonlinear} to LSNMs, and implement a variant of \ourmethod following the RESIT approach (regression with subsequent independence testing)~\citep{peters:14:resit}. For both strategies, we show under which assumptions consistent estimation of cause and effect is attainable with our estimator based on feature maps.

We evaluate all variants of \ourmethod on standard cause-effect benchmark datasets and observe that \ourmethod achieves state-of-the-art overall performance and almost perfect accuracy on datasets which are in line with our assumptions---i.e., Gaussian ANMs and LSNMs.

\ifnotblind
For reproducibility, we make our code publicly available\footnote{\urltocode} and provide all proofs in the Supplementary Material.
\else
\fi

\section{Identifiability of LSNMs}
\label{sec:identifiability}

In this section, we focus on the identifiability of location-scale noise models (LSNMs). A causal model is said to be identifiable under a set of structural constraints, if only the forward (causal) model is well specified and no backward model fulfilling these structural constraints exists.

To formally analyze this problem, we first need to define our assumed causal model.

\begin{definition}[Location-Scale Noise Model]
\label{def:lsnm}
Given two independent random variables $X$ and $N_Y$. If the effect $Y$ is generated by a location-scale noise model, we can express $Y$ as an SCM of the form
\begin{equation}
    Y := f(X) + g(X) N_Y \; ,
\end{equation}
where $f{:} \, \mathcal X \to \mathbb{R}$ and $g{:} \, \mathcal X \to \mathbb{R}_{+}$, i.e.~$g$ is strictly positive.
\end{definition}

We provide an illustration of data generated by a location-scale noise model in Fig.~\ref{fig:lsnm-example}.
LSNMs simplify to ANMs when $g(X)$ is constant, and to multiplicative noise models when $f(X)$ is constant.

To prove identifiability of such a restricted SCM, it is common to derive an ordinary differential equation (ODE), which needs to be fulfilled such that a backward model exists, see e.g.~\citet{hoyer:09:nonlinear}, or \citet{zhang:09:post-nonlinear}. Intuitively, the solution space of such an ODE specifies all cases in which the model is non-identifiable, leaving all specifications which do not fulfill the ODE as identifiable. In the following theorem, we derive such a differential equation for LSNMs and discuss its implications.

\begin{restatable}{theorem}{thidentifiability}\label{th:identifiability}
Assume the data is such that a location-scale noise model can be fit in both directions, i.e.,
\begin{align*}
Y &= f\left(X\right) + g\left(X\right)N_Y, \quad X \Indep N_Y\\
X &= h\left(Y\right) + k\left(Y\right)N_X, \quad Y \Indep N_X.
\end{align*} 
Let $\nu_1\left(\cdot\right)$ and $\nu_2\left(\cdot\right)$ be the twice differentiable log densities of $Y$ and $N_X$ respectively. For compact notation, define 
\begin{align*}
\nu_{X\vert Y}\left(x\vert y\right)&=\log\left(p_{X\vert Y}\left(x\vert y\right)\right) \\
&=\log\left(p_{N_X}\left(\dfrac{x-h\left(y\right)}{k\left(y\right)}\right)/k\left(y\right)\right)\\
&=\nu_2\left(\dfrac{x-h\left(y\right)}{k\left(y\right)}\right)-\log\left(k\left(y\right)\right) \quad \text{and}\\
G\left(x,y\right) & = g\left(x\right)f'\left(x\right)+ g'\left(x\right)\left[y-f\left(x\right)\right]
.
\end{align*}
Assume that $f\left(\cdot\right)$, $g\left(\cdot\right)$, $h\left(\cdot\right)$, and $k\left(\cdot\right)$ are twice differentiable.
Then, the data generating mechanism must fulfill the following PDE for all $x, y$ with $G\left(x,y\right) \neq 0$.
\begin{align}\label{eq:PDEalt}
& 0 = \nu_1''\left(y\right) + \dfrac{g'\left(x\right)}{G\left(x,y\right)}\nu_1'\left(y\right) + \dfrac{\partial^2}{\partial y^2}\nu_{X\vert Y}\left(x\vert y\right)+\\
& \dfrac{g\left(x\right)}{G\left(x,y\right)}\dfrac{\partial^2}{\partial y \partial x}\nu_{X\vert Y}\left(x\vert y\right)+\dfrac{g'\left(x\right)}{G\left(x,y\right)}\dfrac{\partial}{\partial y}\nu_{X\vert Y}\left(x\vert y\right).
\end{align}
\end{restatable}
The equality derived in Theorem~\ref{th:identifiability} is equivalent to the result concurrently provided by \citet{strobl:22:grci} up to fixing a sign error in the terms involving $g'\left(x\right)$. We derived this result independently using a different proof technique and additionally note that $p_{X\vert Y}\left(x\vert y\right)$ cannot be written as univariate function with argument $\left(\left[x-h\left(y\right)\right] / k\left(y\right)\right)$ if $Y \rightarrow X$ is an LSNM with non-constant $k\left(\cdot\right)$.

The conclusion of \citet{strobl:22:grci} is that if we have $x_0$ such that $G\left(x_0, y\right)\neq 0$ for all but countably many $y$, then knowing $\nu_{X\vert Y}\left(x_0\vert y\right)$, $G\left(x_0, y\right)\neq 0$, $g\left(x_0\right)$ and $g'\left(x_0\right)$ leads to $\nu_1\left(y\right)$ being constrained to a two dimensional affine space as \eqref{eq:PDEalt} becomes an ODE. This is in analogy to the result of \citet{hoyer:09:nonlinear} for ANMs. For this case, \citet{zhang:09:post-nonlinear} have refined the result and provide a list of all possible cases of unidentifiable models: only for specific choices of $f\left(\cdot\right)$ and $\nu_2\left(\cdot\right)$, one can find $\nu_1\left(\cdot\right)$ such that the model is invertible. 

This conclusion carries over to the LSNM. 
Assume there exist different values $x$ such that $G\left(x, y\right)\neq 0$ for all but countably many $y$. If $g\left(\cdot\right)$ is strictly positive and $f\left(\cdot\right)$ is injective, this applies to all $x \in \mathbb{R}$ except for at most countably many. Each such value leads to a different ODE in $y$ when plugging it into Eq.~\eqref{eq:PDEalt}. Only when the solution spaces of all ODEs overlap such that the same $\nu_1\left(\cdot\right)$ is found, which must also be valid log-density, the model can be invertible. This is not the case for generic combinations of $g\left(\cdot\right)$, $G\left(\cdot, \cdot\right)$ and $\nu_{X\vert Y}\left(\cdot \vert \cdot \right)$ but only for very specific exceptions. Thus, apart from some pathological cases, an LSNM cannot be invertible. A precise characterization of these cases as in \cite{zhang:09:post-nonlinear} for the post-nonlinear model, which involves the ANM as a special case, has not yet been found for LSNM to the best of our knowledge. 

To provide a bit more intuition regarding the assumptions of Theorem~\ref{th:identifiability}, note that the results only apply to random variables $X$ with unbounded support. This is implied by requiring that the log-density of $N_X$ has to be twice differentiable. For example, $X$ could not follow a uniform distribution. This also implies that $g( \cdot )$ has to be a non-linear (or constant) function since otherwise $g( \cdot )$ is negative for some attainable values of $X$ and does not strictly map to  $\mathbb{R}_+$ as required by our assumptions. 
\ifalternativeversion
Assuming that the noise variable is Gaussian, necessary conditions for the distributions of $X$ and $Y$ as well as the functions $f\left(\cdot\right)$, $g\left(\cdot\right)$, $h\left(\cdot\right)$, and $k\left(\cdot\right)$ can be found~\citep{khemakhem:21:autoregressive}. For completeness, we provide the corresponding result as Theorem~\ref{th:identifiability-autoregressive-gaussian} in Supplementary Material~\ref{app:estimation}.
\fi

\section{Estimation of LSNMs}
\label{sec:consistent-estimation}

To separate the independent noise source $N_Y = \frac{Y-f(X)}{g(X)}$ from $Y$, we derive a consistent maximum likelihood estimator for the log-likelihood of $Y$ given $X$ and a parameter vector $\vtheta$ that models $f(x)$ and $g(x)$. 
The typical parametrization then attempts to fit $f(x)$ as the mean of a Gaussian and $g(x)$ as the standard deviation, such as iterative feasible generalized least-squares~\citep{harvey1976estimating, amemiya1985advanced}. 
However, this leads to non-concavity of the associated log-likelihood and therefore massively complicates consistency of the estimator. In particular, we would define $\theta_1(x) = \mu(x)$ and $\theta_2(x) = \sigma^2(x)$ and have the Gaussian log-likelihood $\log p(y | x, \vtheta) = \log \mathcal{N}(y| \theta_1(x), \theta_2(x))$. Then, differentiating twice with respect to $\theta_2(x)$, we have
\begin{equation}
    \frac{\partial^2 \log \mathcal{N}(y| \theta_1(x), \theta_2(x))}{\partial \theta_2(x)^2} = \frac{\theta_2(x) - 2 (\theta_1(x) - y)^2}{2 \theta_2(x)^3},
\end{equation}
which is negative if $\theta_2(x) < 2(\theta_1(x) - y)^2$ since $\theta_2(x) > 0$ as it models the variance.
Therefore, the Hessian can be indefinite and the log-likelihood cannot be guaranteed to be concave.
With a non-concave estimator, deriving any kind of consistency claims is challenging~\citep{wooldridge2015introductory}. 
Consistent estimation, however, is crucial to guarantee that an estimator can provably identify the causal direction.

To achieve consistent estimation and consequently provide an estimator that can identify the causal direction in the large-data limit, we choose to parameterize the Gaussian with its natural parameters, $\eta_1(x), \eta_2(x)$ with the inverse mapping $\mu(x) = - \tfrac{\eta_1(x)}{2 \eta_2(x)}$ and $\sigma^2(x) = - \tfrac{1}{2 \eta_2(x)}$. 
Due to properties of the natural parametrization of exponential families, it is clear that the Gaussian log-likelihood
\begin{equation}
    \small
    \label{eq:log-likelihood-nat}
    \log p(y | x, \veta(x)) {=} c + \veta(x)\transpose \begin{bmatrix} y \\ y^2 \end{bmatrix} + \frac{\eta_1(x)^2}{4 \eta_2(x)} + \log \sqrt{-2 \eta_2(x)}
\end{equation}
with $c = - \nicefrac{1}{2} \log \left( 2 \pi \right)$, is strictly concave in $\eta_1(x), \eta_2(x)$ \citep{brown1986fundamentals} with constraint $\eta_2(x)<0$. However, the parametrization of the natural parameters through learnable functions matters.

We are particularly interested in non-linear Gaussian LSNMs.
We therefore assume to know the right feature maps giving rise to ground truth $\eta_1(x)$ and $\eta_2(x)$. 
With (non-linear) feature maps $\vpsi(x)\in \R^D$, $\vphi(x) \in \R_{+}^D$, and parameters $\vw_1 \in \R^D, \vw_2 \in \R^D_+$, we have
\begin{equation}
    \eta_1(x) = \vpsi(x)\transpose \vw_1 \;\; \textrm{and} \;\; \eta_2(x) = -\vphi(x)\transpose \vw_2
    \label{eq:feature_map_model}
\end{equation}
with corresponding log-likelihood as $\log p(y|x,\vw)$.
The second natural parameter, $\eta_2(x)$, is ensured to be negative due to positive feature map and parameters.
The log-likelihood of this model is then jointly concave in $\vw_1, \vw_2$ due to the composition of a concave log-likelihood with a linear model and constraint $\vw_2 > 0$.

\begin{restatable}{lemma}{lemconcavity}\label{le:concavity}
With natural parameters modelled as $\eta_1(x) = \vpsi(x)\transpose \vw_1$ and $\eta_2(x) = -\vphi(x)\transpose \vw_2$, the Gaussian log-likelihood function in Eq.~\eqref{eq:log-likelihood-nat} is jointly concave in $\vw_1, \vw_2$.
\end{restatable}

The result relies primarily on the concavity of exponential family distributions in their natural form and has been used in a similar context for heteroscedastic Gaussian processes~\citep{le:05:heteroscedastic}.
In contrast, the commonly used iterative feasible generalized least-squares (FGLS) method for heteroscedastic regression is formulated with a non-concave loss. %
For example, \citet{cawley2004heteroscedastic} use a similar formulation using feature maps but their model is not jointly concave because they do not use natural parameters.

To achieve consistent estimation, we need the log-likelihood to be strictly concave (to guarantee global identification) when given a certain number of samples $T > 0$. This condition can be formulated in terms of the Fisher information, which is the expectation of the negative log-likelihood Hessian under the predictive sample for each data point.

\begin{assumption}\label{as:fisher-information}
There exists $T_0 > 0$, s.t.~for any $\SN \ge \SN_0$ the Fisher information matrix of the Gaussian log-likelihood in Eq.~\eqref{eq:log-likelihood-nat} parameterized as in Eq.~\eqref{eq:feature_map_model}, i.e.,
\begin{equation}\label{eq:empirical-fisher}
I_\SN(\vw) = - \frac{1}{\SN} \sum_{\Sn = 1}^\SN \E_{y \sim p(y|x_\Sn;\vw)} \left[ \nabla^2_\vw \log p(y | x_\Sn; \vw) \right]
\end{equation}
is positive definite with respect to any $\vw \in \mathbb{R}^D \times \mathbb{R}^D_+$.
\end{assumption}

In our case, the Fisher information and the Hessian coincide since we have a generalized linear model with exponential family likelihood~\citep[Sec.~9.2 in][]{martens2020new}. 
Therefore, the Fisher information is guaranteed to be at least positive semi-definite according to Lemma~\ref{le:concavity}.
In particular, each summand in Eq.~\eqref{eq:empirical-fisher} will have rank two.
Thus, our assumption essentially requires that, upon observing enough data points, these summands span different sub-spaces of $\R^{2D}$ such that their sum is full rank.
For sensible feature maps and non-degenerate distributions on $x$, we expect this assumption to hold.
This is trivially the case for $D=2$, which could, for example, be a linear model with $\phi(x)=x$ and $\psi(x)=|x|$, because each summand is already full rank. 
\begin{restatable}{theorem}{thconsistency}\label{th:consistent}
Let $\{(y_i,x_i)\}_{i=\Sn,\dots,\SN}$ be an iid sample with conditional density $p(y|x_\Sn;\vw^*)$ as defined through Eqs.~\eqref{eq:log-likelihood-nat} and~\eqref{eq:feature_map_model}, with $\vw^* \in \mathbb{R}^D \times \mathbb{R}^D_+$, and correctly specified (non-linear) feature maps $\vpsi(x)\in \R^D$ and $\vphi(x) \in \R_{+}^D$.
Further, there exists a $\SN_0 > 0$ for which Assumption~\ref{as:fisher-information} holds, and suppose the usual smoothness criteria hold:
\begin{enumerate}[i)]
    \item\label{as:th-smoothness-1} Derivatives up to second order can be passed under the integral sign in $\int dP(y | x, \vw)$.
    \item\label{as:th-smoothness-2} Third derivatives $\nabla^3 \log p(y | x, \vw)$ are bounded by a function $M(y|x)$ on $\mathbb{R}^D \times \mathbb{R}^D_+$, s.t.
    \[
    \sup_{\theta \in \mathbb{R}^D \times \mathbb{R}^D_+} | \nabla^3 \log p(y | x, \vw)_{jkl} | \le M(y|x) 
    \]
    for all $j,k,l$, and $\mathbb{E}[M(y|x)] < \infty$.
\end{enumerate}
Then, for $\SN \ge \SN_0$ the maximum likelihood estimate $\hat{\vw}_\SN$ is weakly consistent, i.e., for $\SN \to \infty$, $\hat{\vw}_\SN \stackrel{p}{\to} \vw^*$. Further, $\hat{\vw}_\SN$ is asymptotically normal in a sense that $\sqrt{\SN}(\hat{\vw}_\SN - \vw^*) \stackrel{d}{\to} \mathcal{N}(0,I(\vw^*)^{-1})$.
\end{restatable}

The key assumptions for this Theorem to hold are that the feature maps are correctly specified and that they fulfill Assumption~\ref{as:fisher-information}, which we discussed above. The remaining assumptions ensure the smoothness of the log-likelihood and are needed to prove the theorem following the approach of \citet{cramer:46:methods-stats}. There exists follow-up work, which shows that these smoothness criteria can be relaxed, e.g., \citet{kulldorff:57:conditions-consistency} and \citet{wald:49:consistency}, however, we used the conditions provided by \citet{cramer:46:methods-stats} as they are more intuitive. Last, we can also achieve consistency (but not asymptotic normality) by substituting the smoothness criteria in Theorem~\ref{th:consistent} with the condition
\[
\mathbb{E}[|\log p(y|x,\vw)|] < \infty \; ,
\]
as discussed in~\citep[Ch.~15]{wooldridge2015introductory}.

In Supplementary Material~\ref{app:het_reg_algos}, we describe an algorithm to estimate the model parameters.
Due to the concavity in $\vw$, the algorithm can make use of closed-form iterations and reliably converges to a global optimum.
In general however, we expect our model to be misspecified and the underlying feature maps to be unknown.
In this case, we use approximate feature maps for our linear model. 
For example, we use spline-based feature maps in our experiments~\citep{eilers1996flexible}.
In Supplementary Material~\ref{app:estimator_benchmark}, we compare the performance of our proposed estimator to the commonly used iterative feasible generalized least squares (IFGLS) algorithm, which does not rely on a jointly concave log-likelihood formulation. We find that our estimator clearly improves upon IFGLS. Especially, for small datasets, which are common in causal inference, and large numbers of feature maps it leads to significantly less overfitting.

In the case where we do not know the true feature maps, we can alternatively use neural networks to learn them. %
Parameterized by weights $\vw \in \R^D$, we have neural network $\vf: \R \times \R^D \rightarrow \R^2$ that maps from a scalar input to the two natural parameters of the heteroscedastic likelihood.
In particular, we model
\begin{equation}
\label{eq:nn_output}
    \eta_1(x) = f_1(x, \vw) \;\; \textrm{and} \;\; \eta_2(x) = - \tfrac{1}{2} \exp (f_2(x, \vw)),
\end{equation}
which uses an additional exponential link function that ensures positivity while maintaining differentiability.
For details, see Supplementary Material~\ref{app:het_reg_nn}.

\section{Cause-Effect Inference}
\label{sec:ce-inference}

After studying identifiability and estimation of LSNMs, we now propose two strategies to instantiate our estimators: a likelihood-based estimator and an independence-based estimator. Both assume we can separate the independent noise $N_Y$ from $X, Y$. Hence, we first show under which conditions maximizing Eq.~\eqref{eq:log-likelihood-nat} is equivalent to minimizing the dependence between estimated noise $\hat{N}$
and $X$.

One way to measure the dependence between two continuous random variables is mutual information
\begin{equation}
\label{eq:diff-mi}
	\I(X;Y) = \int p(x,y) \log \frac{p(x,y)}{p(x)p(y)} dx dy \; ,
\end{equation}
which is zero iff $X$ and $Y$ are independent. Suppose there exists an invertible mapping from $(x,n_Y)^T$ to $(x,y)^T$, where the mapping function $\vf$ is defined as an SCM with model class $\fModels$, and $\vf$ can be parameterized by $\vtheta$.
For such an invertible mapping with corresponding Jacobian transformation matrix $\bm{J}_{X \to Y}$, mutual information (sample version) of $X$ and $N_Y$ can be expressed as~\citep{zhang:15:loglikelihood}
\begin{equation}
\label{eq:mi-x-noise-sample}
\I\left(\vx, \vn_Y | \vtheta \right) =  - \frac{1}{\SN} \sum_{\Sn=1}^\SN \log \left( \frac{p(x_\Sn)p(n_{Y,\Sn} | \vtheta)}{p_{\fModels}(x_\Sn,y_\Sn) |\bm{J}_{X \to Y}|} \right) \; ,
\end{equation}
where $(\vx, \vn_Y) = \{ (x_i, n_{Y,i}) \}_{i=1, \dots, \SN}$ is an iid sample of size $\SN$, $p_{\fModels}(x,y)$ is the density induced by the SCM with model class $\fModels$, and $|\bm{J}_{X \to Y}|$ is the absolute value of the determinant of $\bm{J}_{X \to Y}$. To arrive at Eq.~\eqref{eq:mi-x-noise-sample}, we exploit that
\[
p(x,n_Y) = p_{\fModels}(x,y) |\bm{J}_{X \to Y}| \; .
\]
 \citet{zhang:15:loglikelihood} note that for ANMs, $|\bm{J}_{X \to Y}| = 1$ can be ignored. For LSNMs, however, $|\bm{J}_{X \to Y}|$ evaluated at $(x,n_Y)$ is equal to $|g(x)|$ and thus, cannot be ignored.

If $p_{\fModels}(x,y)$ is induced by a Gaussian LSNM, the conditional log-density of $y$ given $x$ as defined in Eq.~\ref{eq:log-likelihood-nat} already includes the scaling through $|g(x)|$ since it can be written as $\log p(n_Y | \vtheta) - \log |g(x)|$.
Moreover, $p(x)$ and $p_{\fModels}(x,y)$ are independent of $\vtheta$. The mutual information in Eq.~\eqref{eq:mi-x-noise-sample} is thus minimized as $\log p(y | x, \vtheta)$ is maximized. Moreover, \citet{zhang:15:loglikelihood} show that
\[
\I(\vx, \hat{\vn}_Y | \hat{\vtheta}_\SN) \stackrel{p}{\to} 0
\]
as $\SN \to \infty$, for any consistent MLE $\hat{\vtheta}_\SN$. For LSNMs our estimator derived in Sec.~\ref{sec:consistent-estimation} fulfills this condition.

\subsection{Likelihood-Based Inference}
\label{sec:likelihood-based-inference}
Cause-effect inference via maximum likelihood or related measures such as through the Minimum Description Length principle or Bayesian estimators has proven to be a successful approach for cause-effect inference~\citep{mooij:10:mml,buhlmann:14:cam,marx:17:slope}. The general approach is to compare the log-likelihood estimate (or a related score) for the presumed causal direction $X \to Y$,
\begin{equation}
\label{eq:estimator-ll}
	\ell_{X \to Y}(\hat{\vtheta}_\SN) =  \sum_{i=1}^\SN \log \left( p(x_\Sn) p\left(y_\Sn | x_\Sn, \hat{\vtheta}_\SN \right) \right) \; ,
\end{equation}
to the corresponding score in the $Y \to X$ direction, i.e, $\ell_{Y \to X}(\hat{\vxi}_\SN)$.\!\footnote{We drop the arguments, whenever clear from context.} Thus, we estimate that $X$ causes $Y$ if $\ell_{X \to Y} > \ell_{Y \to X}$, $Y$ causes $X$ if $\ell_{X \to Y} < \ell_{Y \to X}$, and do not decide if both terms as equal. 

When closely inspecting Eq.~\eqref{eq:estimator-ll}, we observe---similar to the ANM case~\citep{zhang:15:loglikelihood}---that comparing $\nicefrac{1}{T} \, \ell_{X \to Y}$ to $\nicefrac{1}{T} \, \ell_{Y \to X}$ for LSNMs is equivalent to comparing $\I(\vx, \hat{\vn}_Y | \hat{\vtheta}_\SN)$ to $\I (\vy, \hat{\vn}_X | \hat{\vxi}_\SN)$ since $p_{\fModels}(x,y)$ is a constant appearing on both sides. Thus, maximum likelihood can be used to identify cause and effect for an LSNM, whenever the model is identifiable as discussed in Sec.~\ref{sec:identifiability} and the estimator is consistent (Sec.~\ref{sec:consistent-estimation}). We formalize these points in Theorem~\ref{th:consistency-mle-approach}, which is adapted from \citet{zhang:15:loglikelihood}.

\begin{restatable}{theorem}{consistencymleapproach}[\citet{zhang:15:loglikelihood}]\label{th:consistency-mle-approach}
Let $P_{\fModels}(X,Y)$ be the joint distribution induced by an LSNM according to Def.~\ref{def:lsnm} with $N_Y {\sim} \mathcal{N}(0,1)$, parametrization $\vtheta$, and consistent MLE $\hat{\vtheta}_\SN$. Given an iid sample $(\vx, \vy) \sim P_{\fModels}(X,Y)$ of size $\SN$,
\begin{equation}
\label{eq:identifiability-mle}
\lim_{\SN \to \infty} \ell_{X \to Y}(\hat{\vtheta}_\SN) - \ell_{Y \to X}(\hat{\vxi}_\SN) \ge 0 \; ,
\end{equation}
with equality, if and only if, a backward model parameterized by $\vxi$ exists according to Theorem~\ref{th:identifiability-autoregressive-gaussian}, and $\ell_{Y \to X}(\hat{\vxi}_\SN)$ converges to $\ell_{Y \to X}(\vxi)$ as $\SN \to \infty$.
\end{restatable}

Clearly, if we know $p(x)$ and all assumptions in Theorem~\ref{th:consistency-mle-approach} are fulfilled, our estimator based on non-linear feature maps is suitable to consistently identify cause and effect according to Theorem~\ref{th:consistency-mle-approach}. In practice, we assume a fully Gaussian model, meaning that we also model $p(x)$ as Gaussian. It would, however, be straightforward to use a different prior.

\subsection{Independence-Based Inference}
\label{sec:indep-based-inference}
As a second approach, we apply regression with subsequent independence testing (RESIT)~\citep{peters:14:resit}.
Due to the additional degree of freedom to fit the scale of an LSNM, we expect our estimators to be beneficial in this setting.

That is, we first estimate the residuals $\hat{N}_Y = \frac{Y - \hat{f}(X)}{\hat{g}(X)}$ using one of the estimators proposed in Sec.~\ref{sec:consistent-estimation}, and subsequently test whether $\hat{N}_Y$ is independent of the presumed cause $X$. We repeat these steps for the $Y \to X$ direction and decide for that direction with the larger $p$-value; i.e., the smaller the $p$-value, the more evidence for rejecting independence. If both $p$-values are insignificant, both directions admit an LSNM, and one could decide to not force a decision.

A possible problem with such a RESIT approach is that the estimated residuals may be inherently dependent on the input~\citep{kpotufe:14:consistency-anm,mooij:16:pairs}, when we do not use a suitable estimator. As discussed by~\citet{mooij:16:pairs}, a regression estimator is suitable, if i) it is $L_2$ consistent, or ii) it is weakly consistent, but we perform sample splitting---i.e., we estimate the function parameters on one half of the data and estimate the dependence between residuals and potential cause on the other half of the data. Further, \citet[Thm.~20]{mooij:16:pairs} proved that Hilbert-Schmidt independence criterion (HSIC)~\citep{gretton:05:hsic} is consistent given such a suitable estimator.%
\!\footnote{These results have also been extended beyond the two-variable case by~\citet{pfister:18:dhsic}.} Therefore, we chose HSIC for our empirical evaluation.

$L_2$-consistency is a stronger statement than what we proved in Theorem~\ref{th:consistent}. For our estimator, we would need to show that the Fisher information is sufficiently peaked at $\vtheta$ to achieve $L_2$ consistency, which is difficult for arbitrary non-linear feature maps. Even though, our estimator is suitable for Gaussian LSNMs if we perform sample splitting, we observe a substantially improved performance when using the full dataset for training and independence testing.

\subsection{Location-Scale Causal Inference}

To instantiate the estimators above, we propose \ourmethod (\textbf{lo}cation-scale \textbf{c}ause-effect \textbf{i}nference), which follows a simple two-step procedure: We first standardize the data to avoid any bias that might be induced by the scale of the involved variables~\citep{reisach:21:beware}.\!\footnote{Since we assume a Gaussian LSNM, the marginals for $X$ and $Y$ cancel when comparing $\ell_{X \to Y}$ to $\ell_{Y \to X}$ after standardization.} Then, we use one of the estimators proposed in Sec.~\ref{sec:consistent-estimation} to estimate $f$ and $g$ for the $X \to Y$ direction, and $h$ and $k$ for the $Y \to X$ direction. Last, we determine the causal direction, either by choosing the direction with the higher likelihood, as described in Sec.~\ref{sec:likelihood-based-inference} (\ourmethodmle), or the one with the larger $p$-value provided by HSIC, as proposed in Sec.~\ref{sec:indep-based-inference} (\ourmethodhsic).

We expect the likelihood-based approach to perform best for Gaussian ANMs and LSNMs, since it is consistent in these settings, while we expect \ourmethodhsic to have an advantage when the data is non-Gaussian distributed.

\section{Related Work}
\label{sec:related}

Cause-effect inference is a well-studied problem, and the first identifiability results have been established for linear additive noise models~\citep{shimizu:06:lingam}, where the effect $Y$ can be expressed as a linear function of the cause $X$ and an additive non-Gaussian error term. It has been proven that the backward model does not exist for a broad range of ANMs~\citep{hoyer:09:nonlinear,peters:12:ifmoc,shoubo:18:mixture,peters:17:book,buhlmann:14:cam}, and consistency results for ANMs have been derived~\citep{kpotufe:14:consistency-anm}.
Further, these results on identifiability have been extended to post non-linear causal models~\citep{zhang:09:post-nonlinear} and a link to maximum likelihood estimation has been established~\citep{zhang:15:loglikelihood}.

A line of research that is not based on assumptions on the SCM builds upon the principle of independent mechanisms and postulates that the mechanism ($p_{Y|X}$) is independent of the cause ($p_X$), whereas this does not hold for the anti-causal direction~\citep{janzing:10:algomarkov,peters:17:book}. Besides being a postulate, the principle of independent mechanism can also be linked to known identifiability results, e.g., to the likelihood approach we presented in Sec.~\ref{sec:ce-inference}. For Gaussian additive noise models for example, $p_{Y|X}$ models the residual distribution which is independent of $p_X$ for correctly specified models. Note, however, that for location-scale noise models, the definition has to be revisited and we cannot consider $p_{Y|X}$ as an independent mechanism, since the independent noise source is equal to $p_{Y|X} \cdot |\bm{J}_{X \to Y}|$.\!\footnote{A related concept with identifiability results for discrete data is entropic causal inference~\citep{kocaoglu:17:entropic}.} Approaches motivated by this postulate either try to approximate the involved distributions directly~\citep{janzing:12:igci,sgouritsa:15:cure,goudet:18:learning}, or aim to approximate its information-theoretic counterpart: the algorithmic independence of conditionals postulate~\citep{janzing:10:algomarkov}, which compares the Kolmogorov complexities~\citep{kolmogorov:65:information} of the involved factorizations. Since Kolmogorov complexity is not computable, it is typically approximated via Minimum Description Length or Minimum Message Length~\citep{kalainathan:19:generative,marx:17:slope,mitrovic:18:causalkernel,mooij:10:mml,mian:21:globe}. In a broader sense, these approaches also link to (penalized) likelihood methods or Bayesian structure learning~\citep{peters:14:equal-error,buhlmann:14:cam,marx:19:sloppy,lorch:21:dibs}, which are of increasing interest, especially for graphs beyond the bivariate setting~\citep{zheng:18:notears,vowels:21:like-dags}.

\begin{figure*}[t!h!]
\centering
\includegraphics{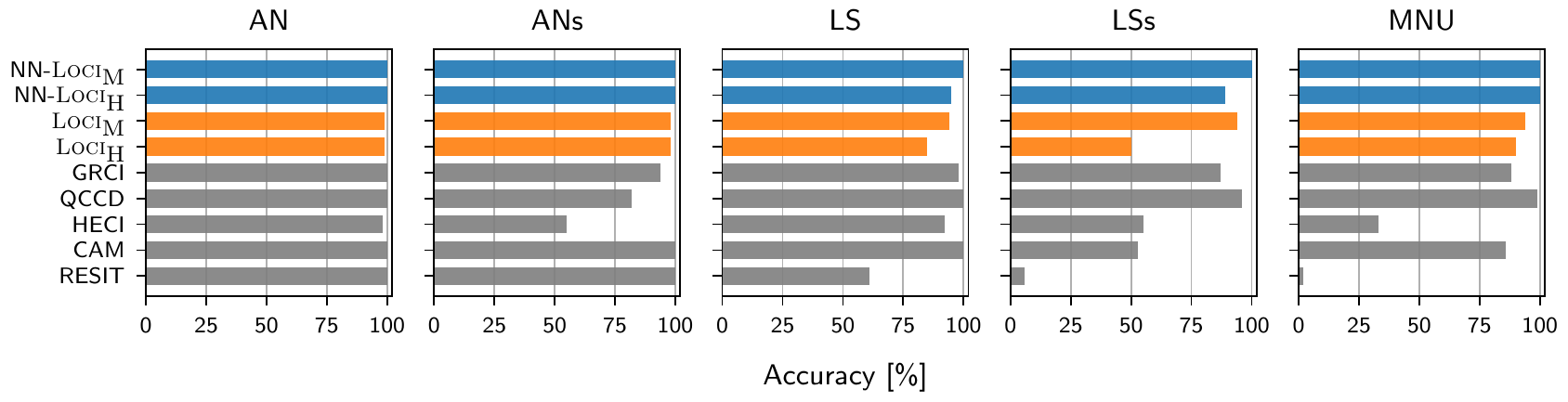}
\vspace{-2.5em}
\caption{Performance of \ourmethod with maximum likelihood and HSIC estimator using either neural network 
(\protect\tikz[baseline=-0.15ex,inner sep=0pt]{\protect\draw[line width=6pt, C1light] (0,0.05) -- ++(0.35,0)})
or linear regressor 
(\protect\tikz[baseline=-0.15ex,inner sep=0pt]{\protect\draw[line width=6pt, C2light] (0,0.05) -- ++(0.35,0)})
in comparison to the baselines 
(\protect\tikz[baseline=-0.15ex,inner sep=0pt]{\protect\draw[line width=6pt, Cglight] (0,0.05) -- ++(0.35,0)})
on benchmark datasets by \citet{tagasovska:20:bqcd}.
On these datasets, our model is either well-specified or slightly misspecified and therefore \ourmethodmle performs best.
\ourmethodmle with neural network regressor makes no error on these datasets while other methods designed for similar data misclassify some pairs.
}
\label{fig:anlsmnu_barplot}
\vspace{-0.5em}
\end{figure*}

For location-scale noise models, several independent identifiability results exist. \textsc{HECI}~\citep{xu:22:heci} and \textsc{FOM}~\citep{cai:20:fourth-order} are two proposals, which extend the identifiability results for ANMs~\citep{blobaum:18:reci}, which focus on the low-noise regime. In particular, \citet{xu:22:heci} develop an approach that partitions the domain space and develop a loss based on the geometric mean between error terms, and \citet{cai:20:fourth-order} suggest a score based on fourth-order moments, which they estimate via heteroscedastic Gaussian process regression. Approaches without identifiability results have also shown strong empirical performance on location-scale benchmark datasets, such as \textsc{QCCD}~\citep{tagasovska:20:bqcd}, based on non-parametric quantile regression, and estimators for conditional divergences~\citep{fonollosa:19:challenge-winner,duong:21:cdci}. Most related to our approach is the concurrent work of~\citep{strobl:22:grci}, who prove Theorem~\ref{th:identifiability} independently of us via a different technique as discussed in Sec.~\ref{sec:identifiability}. \citet{strobl:22:grci}, however, focus on root cause analysis and use a heuristic two stage algorithm, \textsc{GRCI}, to learn the mean and variance via cross validation. Subsequently, they estimate the degree of dependence via mutual information.

In contrast to previous work, we not only provide identifiability results for LSNMs, but also bridge the gap and provide the first consistent estimator for causal LSNMs. 
Further, we find that our LSNM maximum-likelihood estimator can outperform iterative feasible generalized least-squares~\citep[IFGLS,][]{harvey1976estimating}, the de-factor standard for heteroscedastic models, especially on small data~(cf.~Supplementary Material~\ref{app:estimator_benchmark}).
Additionally, we provide a neural network-based approach to learn LSNMs, which exhibits state-of-the-art performance on a variety of benchmark tasks.

\section{Experiments}

We empirically compare \ourmethod to state-of-the-art bivariate causal inference methods and study the benefit of modelling location-scale noise as well as a post-hoc independence test.
The performance of bivariate causal inference methods is assessed in terms of accuracy and area under the decision rate curve (AUDRC).\!\footnote{We choose AUDRC over AUPR and AUROC, since it weights correctly identified $X \to Y$ pairs in the same way as correctly identified $Y \to X$ pairs and thus avoids an arbitrary selection of true positives and true negatives, which can lead to non-interpretable results on unbalanced datasets~\citep[Table~1]{marx:19:sloper}.}
The accuracy measures the fraction of correctly inferred cause-effect relationships and the AUDRC measures how well the decision certainty indicates accuracy.
The certainty is, for example, indicated by the likelihood or $p$-value difference in both directions.
Thus, a high AUDRC indicates that an estimator tends to be correct when it is certain and only incorrect when it is uncertain.
Mathematically, given the ground truth direction $t(\cdot)$ and a causal estimator for the direction $f(\cdot)$ on $M$ pairs with ordering $\pi(\cdot)$ according to the estimator's certainty, we have
\begin{equation}
    \label{eq:audrc}
    \text{AUDRC} = \tfrac{1}{M} {\textstyle \sum_{m=1}^M \tfrac{1}{m} \sum_{i=1}^{m}} {1}_{f(\pi(i)) = t(\pi(i))}.
\end{equation}
That is, we average the accuracy when iteratively adding the pair the estimator is the most certain about.
Hence, both accuracy and AUDRC are between $0$ and an optimum of $1$.

Overall, we find that \ourmethod performs on par with the best methods for causal pair identification and can achieve perfect accuracy even when the model assumptions are slightly violated.
In the case of additive noise models, our estimators perform on par with custom estimators for these cases.
In the case of location-scale noise models, our method performs often better than previously proposed estimators for this purpose.
We also find that the location-scale noise model estimators we designed greatly improve over corresponding additive noise model estimators with the same structure.
Across the $13$ benchmarks we consider, our method ranks first and improves over concurrent methods for causal inference of location scale noise models like GRCI~\citep{strobl:22:grci} and HECI~\citep{xu:22:heci}.

\paragraph{Baselines.}
We consider CAM~\citep{buhlmann:14:cam}, which is based on homoscedastic maximum likelihood and a corresponding approach based on additional independence tests~\citep[RESIT]{peters:14:resit}.
Further, we compare to three methods that are also designed for heteroscedasticity:
QCCD~\citep{tagasovska:20:bqcd}, which uses quantile regression, HECI~\citep{xu:22:heci}, which uses binning, and GRCI~\citep{strobl:22:grci}, which uses non-linear regression and cross-validation based scores. In Supplementary Material~\ref{app:experiments}, we add CGNN~\cite{goudet:18:learning} and ICGI~\cite{janzing:12:igci} as non-LSNM baselines.

\paragraph{Datasets.}
To assess the performance of our method on datasets where the model assumptions are only slightly violated, we consider the five synthetic datasets proposed by \citet{tagasovska:20:bqcd} that consist of additive (AN, ANs), location scale (LS, LSs), and multiplicative (MNU) noise-models. 
The datasets with suffix `s' use an invertible sigmoidal function making identification more difficult. 
Further, we assess the performance of our approach on a wide variety of common benchmarks where our assumptions are most likely violated.
We consider the Net dataset constructed from random neural networks, the Multi datasets using polynomial mechanisms and various noise settings (including post non-linear noise), and the Cha dataset used for the effect pair challenge~\citep{guyon2019cause}.
Further, we consider the SIM and T{\"u}bingen datasets of the benchmark by \citet{mooij:16:pairs}.

\subsection{Additive, Location-Scale, and Multiplicative Noise}
The five datasets proposed by \citet{tagasovska:20:bqcd} can be modelled using the location-scale noise model that we assume~(Def.~\ref{def:lsnm}) and therefore provide an optimal test case for our methods and theoretical results.
MNU additionally provides an interesting test-case because it is not generated with Gaussian noise.
In Fig.~\ref{fig:anlsmnu_barplot}, we display the accuracy of our estimators using both a linear model with feature maps and a neural network.
We find that our estimators, especially \ourmethodmle relying on the maximum likelihood, achieve almost-perfect performance on this benchmark.
In the following, we compare the performances of the methods in detail.

\paragraph{Comparison to Methods for Additive Noise Models.}
The direct comparison to CAM~\citep{buhlmann:14:cam} and RESIT~\citep{peters:14:resit} is relevant because these two methods rely on additive noise models and can be seen as equivalents to \ourmethodmle and \ourmethodhsic, respectively.
Fig.~\ref{fig:anlsmnu_barplot} shows that our method with a neural network estimator, NN-\ourmethodmle, achieves perfect performance across the five datasets while the method based on feature maps performs only slightly worse.
In particular, \ourmethodmle performs better than the homoscedastic variant CAM and \ourmethodhsic performs significantly better than RESIT.
In the case where the assumptions seem to be aligned with our estimators, no independence test is needed and \ourmethodmle performs better than \ourmethodhsic.
In Fig.~\ref{fig:improvement}, we further show how the LSNM estimators that we propose improve over an ANM estimator using the same feature maps or neural network architecture.
Especially on the three datasets that have complex noise models, LS, LSs, and MNU, our estimators greatly improve the performance.
Across all benchmarked datasets, ANM estimators perform the worst on these three and an additional independence test can further decrease the performance below $10\%$ accuracy. 
Instead, using our LSNM estimators leads to almost perfect accuracy with or without the independence test.

\begin{figure}
    \centering
    \includegraphics{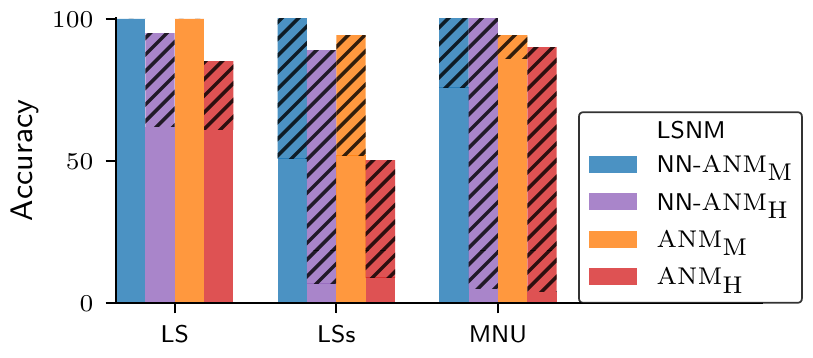}
    \vspace{-2em}
    \caption{Improvement of using our location-scale noise model (LSNM) estimators (striped area) over corresponding additive noise model (ANM) estimators (solid area) with the same network architectures and feature maps.
    The ANM estimators use a Gaussian likelihood with fixed observation noise.
    Only on LS, the likelihood-based ANM estimators are on par with the LSNM variant ($\textrm{ANM}_\textrm{M}$ performing 3\% points better than \ourmethodmle on LS).}
    \label{fig:improvement}
\end{figure}

\paragraph{Comparison to location-scale estimators.}
Although QCCD~\citep{tagasovska:20:bqcd}, GRCI~\citep{strobl:22:grci}, and HECI~\citep{xu:22:heci} are designed in particular for heteroscedastic models, they do not achieve perfect performance on this benchmark like NN-\ourmethodmle.
In particular, these methods struggle on additive noise models with sigmoid non-linearity (ANs) and on datasets with the assumed location-scale noise model (LS) while our proposed method, in particular with a flexible neural network, achieves perfect performance across all these cases.

\subsection{Performance on other Datatsets}
We assess the performance of our method on 8 other benchmark datasets where the LSNM assumptions might be violated.
Despite this, we find that the proposed methods perform on par or better than existing state-of-the-art methods for bivariate causal inference.
In particular our NN-\ourmethodhsic estimator performs best in aggregate accuracy over all benchmarks followed by GRCI~\citep{strobl:22:grci}.
In turn, GRCI~\citep{strobl:22:grci} performs slightly better in terms of AUDRC followed by our method.
Overall, our method performs best on $7$ out of the $13$ benchmarks considered.
The detailed results for all figures can be found in Tables~\ref{app:tab:accuracy} and \ref{app:tab:audrc} in the Supplementary Material.
In comparison to the five benchmarks where our assumptions hold approximately, the additional independence test greatly helps improve performance on the other 8 benchmark sets.

Overall, the results on various benchmarks suggest that \ourmethodmle should be used if one expects the Gaussian noise assumptions to hold.
If the assumptions are likely violated, \ourmethodhsic should be preferred.

\begin{figure}
    \centering
    \includegraphics{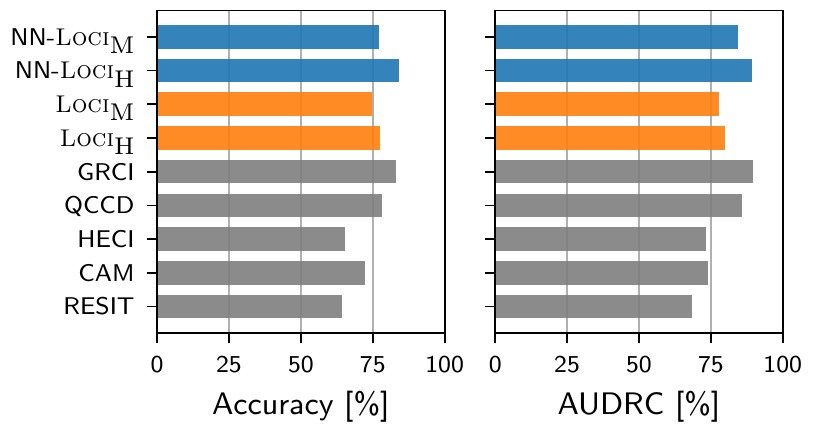}
    \vspace{-2em}
    \caption{Average accuracy and area under the decision rate curve (AUDRC) over all $13$ benchmark datasets. \ourmethod with a neural network performs best overall in terms of accuracy, GRCI in terms of AUDRC.}
    \vspace{-0.5em}
    \label{fig:overall_perf}
\end{figure}

\section{Conclusion}

We considered the problem of cause-effect inference in location-scale noise models from observational data. We proved that the causal direction is identifiable for LSNMs except for some pathological cases, and proposed two empirical estimators for this setting: a concave feature map-based estimator and a neural network-based approach. For cause-effect inference, we instantiated both variants in a likelihood-based framework, as well as performed subsequent independence testing. The likelihood approach has almost perfect accuracy on benchmark data that approximately meets our assumptions, whereas the independence-based approach is more stable when the data is non-Gaussian. Overall, the neural network instantiation has a slight edge over the feature map-based estimator. The latter, however, outperforms comparable methods such as IFGLS and might be of independent interest. 
For future work, it would be interesting to inverstigate the non-Gaussian setting in more detail.

\ifnotblind
\section*{Acknowledgements}
A.I. gratefully acknowledges funding by the Max Planck ETH Center for Learning Systems (CLS). A.M.~is supported by the ETH AI Center with a postdoctoral fellowship. Travel support for A.M.~was provided by ETH Foundations of Data Science (ETH-FDS) at ETH Zurich.
\else
\fi

\bibliography{abbreviations,bib-paper}
\bibliographystyle{icml2023}

\newpage
\appendix
\onecolumn

\section{Heteroscedastic Regression Algorithms}
\label{app:het_reg_algos}

Here we describe the estimators we use for maximum likelihood.
We assume the data are vectors $\vx \in \R^\SN$ and observations $\vy \in \R^\SN$.
Further, we place an uninformative Gaussian prior on the parameter vector $\vw \sim \gauss(\vzero, \delta\inv \mI)$.
Assuming conditional independence, the log joint with parameters $\vw$ is
\begin{equation}
    \log p(\vy, \vw | \vx) = \sum_{\Sn=1}^\SN \log p(y_\Sn | x_\Sn, \vw) + \log \gauss(\vw; \vzero, \delta\inv \mI) \; .
\end{equation}
This ensures that the log joint $\log p(\vy, \vw | \vx)$ is strictly concave for our concave estimator and helps with optimization.
It can further be used to adjust the complexity of the model by adapting $\delta$.
In our case, we use a very small $\delta$ with value $10^{-6}$ which does not regularize but rather helps inverting the Hessian for closed-form updates described below.

\subsection{Concave Linear Heteroscedastic Regression}
\label{app:het_reg_linear}
Optimizing the parameters of the concave heteroscedastic regression model with fixed feature maps can partly be done in a closed-form and partly requires gradient-based optimization.
The parameter $\vw_1$ that controls the first natural parameter through $\eta_1(x) = \vpsi(x)\transpose \vw_1$ can be optimized in closed-form.
We define $\mPsi = [\vpsi(x_1), \ldots, \vpsi(x_\SN)]\transpose \in \R^{\SN \times D}$ as the design matrix.
Differentiating the log-likelihood with respect to $\vw_1$, we have
\begin{equation}
    \nabla_{\vw_1} \log p(\vy , \vw| \vx) = \mPsi\transpose \left(\vy - \valpha \hadamard \mPsi \vw_1\right) - \delta \vw_1,
\end{equation}
where $\valpha = -[\tfrac{1}{2\eta_2(x_1)}, \ldots, \tfrac{1}{2\eta_2(x_\SN)}]\transpose$ denotes variances for each data point.
Therefore, this is easily recognized as a weighted least-squares problem that can be solved as
\begin{equation}
    \vw_1 \leftarrow (\mPsi\transpose \diag(\valpha) \mPsi + \delta \mI)\inv \mPsi\transpose \vy,
\end{equation}
where $\diag(\valpha)$ is a $\SN \times \SN$ diagonal matrix with entries $\valpha$.
While the objective is also concave in $\vw_2$ that controls the second natural parameter $\eta_2(x)$, there is no closed-form solution available.
Therefore, we optimize this part of the objective using the L-BFGS optimizer~\citep{byrd1995limited} with the additional constraint that $\vw_2 \geq 0$.
We then fit the model alternatingly, similar to FGLS, by first using the closed-form solution for $\vw_1$ and then optimizing the second parameter $\vw_2$ for multiple steps using L-BFGS. 
We iterate until the objective value does not improve anymore.

\subsection{Neural Network Heteroscedastic Regression}
\label{app:het_reg_nn}

In general, it is expected that our model is misspecified or that we do not know the underlying feature maps.
To deal with this case, we use neural networks that can automatically learn the feature maps specifically per problem instead of relying on splines or more traditional feature maps.
In contrast to common implementations of neural networks for heteroscedastic regression, which parameterize the mean and variance, we employ the natural parametrization in line with our model based on feature maps.
With neural network $\vf: \R \times \R^D \rightarrow \R^2$ we have
\begin{equation}
    \eta_1(x) = f_1(x, \vw) \;\; \textrm{and} \;\; \eta_2(x) = - \tfrac{1}{2} \exp (f_2(x, \vw)),
\end{equation}
which uses an additional exponential link function to ensure positivity while maintaining differentiability.
Further, above model is sensible because for a single layer a zero-mean Gaussian prior would cause a prior predictive mode at unit variance and hence match the data after standardization.
To optimize the model, we use the Adam optimizer~\citep{kingma2014adam}.

\section{Theoretical Results}
\label{app:estimation}

\thidentifiability*

\begin{proof}
We follow the proof technique of \cite{zhang:09:post-nonlinear}, i.e., we build upon the linear separability of the logarithm of the joint density of independent random variables. That is, for a set of independent random variables whose joint density is twice differentiable, the Hessian of the logarithm of their density function is diagonal everywhere~\citep{lin:98:theorem}. We first define the joint distribution $p(x,n_Y)$ via the change of variable formula, then derive the Hessian of its logarithm, and lastly, derive an PDE which is necessary to hold such that an inverse model can exist.

We define the change of variables from $\left\{x, n_Y\right\}$ to $\left\{y, n_X\right\}$
\begin{align*}
y &= f\left(x\right) + g\left(x\right)n_Y,\\
n_X &= \left[x-h\left(y\right)\right]/k\left(y\right).
\end{align*}
The according Jacobian matrix amounts to
\begin{equation*}
\begin{pmatrix}
\partial y / \partial x & g\left(x\right) \\
\dfrac{1}{k\left(y\right)} - \partial y / \partial x \dfrac{k\left(y\right)h'\left(y\right) + \left[x-h\left(y\right)\right]k'\left(y\right)}{k\left(y\right)^2} & -g\left(x\right) \dfrac{k\left(y\right)h'\left(y\right) + \left[x-h\left(y\right)\right]k'\left(y\right)}{k\left(y\right)^2}
\end{pmatrix},
\end{equation*}
with absolute determinant $g\left(x\right)/k\left(y\right)$ such that
\begin{equation*}
p\left(x, n_Y\right) = \dfrac{g\left(x\right)}{k\left(y\right)} p\left(y, n_X\right).
\end{equation*}
Under independence it holds 
\begin{align*}
\dfrac{\partial^2}{\partial x \partial n_Y} \log\left(p\left(x, n_Y\right)\right) & = 0 \quad \text{such that} \\
\dfrac{\partial^2}{\partial x \partial n_Y} \log\left(\dfrac{g\left(x\right)}{k\left(y\right)} p\left(y, n_X\right)\right) & =\dfrac{\partial^2}{\partial x \partial n_Y} \left[\nu_1\left(y\right) + \nu_2\left(n_X\right) + \log\left(g\left(x\right)\right) - \log \left(k\left(y\right)\right)\right] = 0.
\end{align*}
Evaluating this quantity and dividing by $\left(\partial y / \partial x\right)\left( \partial y / \partial n_Y\right)$ leads to 
\begin{align}
& \nu_1''\left(y\right) + \nu_1'\left(y\right)\left[\dfrac{g'\left(x\right)}{G\left(x,n_Y\right)}\right]+\nu_2''\left(n_X\right)\left[\dfrac{h'\left(y\right) + k'\left(y\right)n_X}{k\left(y\right)^2} \left[h'\left(y\right) + k'\left(y\right)n_X - \dfrac{g\left(x\right)
}{G\left(x,n_Y\right)} \right]\right]+ \nonumber \\
& \nu_2'\left(n_X\right)\left[\dfrac{2 \left[n_X k'\left(y\right)^2 + h'\left(y\right)k'\left(y\right)\right]}{k\left(y\right)^2}-\dfrac{h''\left(y\right)+n_X k''\left(y\right)}{k\left(y\right)}-\dfrac{h'\left(y\right)g'\left(x\right)+n_X k'\left(y\right)g'\left(x\right)}{k\left(y\right)G\left(x,n_Y\right)}\right. -\\
&\left. \dfrac{g\left(x\right)k'\left(y\right)}{k\left(y\right)^2G\left(x,n_Y\right)}\right] + \dfrac{k'\left(y\right)^2-k\left(y\right)k''\left(y\right)}{k\left(y\right)^2}-\dfrac{g'\left(x\right)k'\left(y\right)}{k\left(y\right)G\left(x,n_Y\right)}=0, \label{eq:PDE}
\end{align}
where
\begin{equation*}
G\left(x,n_Y\right) = \dfrac{\partial y}{\partial x} \dfrac{\partial y}{\partial n_Y} = g\left(x\right)\left[f'\left(x\right) + g'\left(x\right)n_Y\right].
\end{equation*}
Assuming injectivity of $f\left(\cdot\right)$ and positivity of $g\left(\cdot\right)$, $G\left(x,n_Y\right)$ can only be $0$ on a set of measure $0$. Finally, plugging in all the definitions in Eq.~\eqref{eq:PDEalt} and taking the derivatives, one finds that Eq.~\eqref{eq:PDEalt} is identical to Eq.~\eqref{eq:PDE}.
\end{proof}

\lemconcavity*

\begin{proof}
To prove concavity of the model, we show that the Hessian of the loss w.r.t. parameters $\vw_1, \vw_2$ is negative semidefinite.
We write $\veta(x) = [\eta_1(x), \eta_2(x)]\transpose \in \R^2$ for the natural parameter vector and thus have the log-likelihood
\begin{equation}
    \log p(y | x, \veta) = - \frac{1}{2} \log \left( 2 \pi \right) + \veta(x)\transpose \begin{bmatrix} y \\ y^2 \end{bmatrix} + \frac{\eta_1(x)^2}{4 \eta_2(x)} + \frac{1}{2} \log (-2 \eta_2(x)) \; .
    \label{eq:app:log_lik}
\end{equation}
Further, we concatenate both parameters and feature maps so that we have $\vw = [\vw_1\transpose \vw_2\transpose]\transpose \in \R^{2D}$ and $\mPhi(x) \in \R^{2D \times 2}$ is a block-diagonal concatenation of feature maps such that the first column is $\vpsi(x)$ followed by zeros and the second column is $D$ zeros followed by $-\vphi(x)$, i.e.,
\begin{equation}
    \mPhi(x) = 
    \begin{pmatrix}
       \vpsi(x) & \vzero \\
       \vzero & -\vphi(x)
    \end{pmatrix}\,.
\end{equation}
We then have $\veta(x) = \mPhi(x)\transpose \vw$ for our natural parameters.
By the chain rule, the Hessian w.r.t. parameters $\vw$ can be written as
\begin{align}\label{eq:second-derivative-w}
    \nabla_{\vw}^2 \log p(y | x, \vw)
    &= [\nabla_\vw \veta(x)]\transpose [\nabla_\veta^2 \log p(y|x,\veta)] [\nabla_\vw \veta(x)] + [\nabla_\vw^2 \veta(x)]\transpose [\nabla_\veta \log p(y|x,\veta(x))] \\
    &= \mPhi(x) [\nabla_\veta^2 \log p(y|x,\veta)] \mPhi(x)\transpose,
\end{align}
because the second derivative of $\veta(x) = \mPhi(x)\transpose \vw$ w.r.t.~$\vw$ is zero due to linearity thus eliminating the second summand.
Because negative definiteness of $\nabla_\veta^2 \log p(y|x,\veta)$ implies negative semidefiniteness for any matrix $\mPhi(x) \nabla_\veta^2 \log p(y|x,\veta) \mPhi(x)\transpose$, it only remains to show that $\nabla_{\vw}^2 \log p(y | x, \vw)$ is indeed negative definite.
While negative definiteness already follows from the natural parameterisation, the Hessian is given by:
\begin{equation}
    \nabla_\veta^2 \log p(y | x, \veta) =
    \begin{pmatrix}
        \tfrac{1}{2 \eta_2(x)} & - \tfrac{\eta_1(x)}{2 \eta_2(x)^2} \\
        - \tfrac{\eta_1(x)}{2 \eta_2(x)^2} & \tfrac{\eta_1(x)^2}{2\eta_2(x)^3} - \tfrac{1}{2 \eta_2(x)^2}
    \end{pmatrix}.
\end{equation}
To show that this matrix is negative definite, we use the determinant criterion.
This requires that odd upper determinants be negative and even ones positive.
The upper left determinant is negative since $\tfrac{1}{2 \eta_2(x)}$ is negative due to $\eta_2(x) < 0$.
The determinant of the entire matrix is given by $- \tfrac{1}{4 \eta_2(x)^3}$ and is positive due to $\eta_2(x) < 0$.
This shows that the log-likelihood Hessian w.r.t. $\eta$ is negative definite and concludes the proof that the log-likelihood is concave in $\vw$.
\end{proof}

\thconsistency*

\begin{proof}
We first establish that the true parameter $\vw^*$ is identifiable. Since the density $p(y | x, \vw^*)$ is part of the exponential family, the parameter space $\mathbb{R}^D \times \mathbb{R}^D_+$ is a convex set, and for $\SN \ge \SN_0$, $I(\vw) > 0$ for all $\vw \in \mathbb{R}^D \times \mathbb{R}^D_+$ (by Assumption~\ref{as:fisher-information}), it follows that any $\vw \in \mathbb{R}^D \times \mathbb{R}^D_+$ is globally identifiable as proven by \citet{rothenberg:71:identification}[Theorem~3]. %

Since there exists a $\SN \ge \SN_0$, so that $I(\vw) > 0$ for any $\vw \in \mathbb{R}^D \times \mathbb{R}^D_+$, the main condition for Cram{\'e}r's consistency proof is also fulfilled~\citep{cramer:46:methods-stats}. Combined with the smoothness conditions \ref{as:th-smoothness-1}) and \ref{as:th-smoothness-2}), we fulfill all criteria according to~\citet{cramer:46:methods-stats} resp.~\citet{liang:84:conditional-like-consistency}, who layed out the assumptions for conditional log-likelihoods such that the claim holds.
\end{proof}

\consistencymleapproach*

\begin{proof}
From Eq.~\eqref{eq:log-likelihood-nat} and Eq.~\eqref{eq:mi-x-noise-sample}, it follows that for LSNMs, 
\begin{align}
	\ell_{X \to Y}(\hat{\vtheta}_\SN) &= \sum_{\Sn=1}^\SN p_{\fModels}(x_\Sn,y_\Sn) - \SN \cdot \I(\vx, \hat{\vn}_Y | \hat{\vtheta}_\SN) \\
	\ell_{Y \to X}(\hat{\vxi}_\SN) &= \sum_{\Sn=1}^\SN p_{\fModels}(x_\Sn,y_\Sn) - \SN \cdot \I(\vy, \hat{\vn}_X | \hat{\vxi}_\SN) \; .
\end{align}
Thus, it suffices to show that 
\begin{equation}
 \lim_{T \to \infty}  \I(\vx, \hat{\vn}_Y | \hat{\vtheta}_\SN) - \I(\vy, \hat{\vn}_X | \hat{\vxi}_\SN) \le 0 \; ,
\end{equation}
with equality, if and only if, a backward model parameterized by $\vxi$ exists according to Theorem~\ref{th:identifiability-autoregressive-gaussian}, and $\ell_{Y \to X}(\hat{\vxi}_\SN)$ converges to $\ell_{Y \to X}(\vxi)$ as $\SN \to \infty$. By consistency of $\hat{\vtheta}_\SN$, $\lim_{\SN \to \infty} I(\vx, \hat{\vn}_Y | \hat{\vtheta}_\SN) = 0$~\citep{zhang:15:loglikelihood}. Further, if and only if a backward model exists s.t.~$Y \Indep N_X$, the mutual information in the backward direction is zero. As a consequence, the empirical estimator approaches zero, if and only if a backward model exist and $\ell_{Y \to X}(\hat{\vxi}_\SN)$ approaches $\ell_{Y \to X}(\vxi)$ as $\SN \to \infty$. Otherwise $\lim_{T \to \infty} \I(\vy, \hat{\vn}_X | \hat{\vxi}_\SN) > 0$.
\end{proof}

\ifalternativeversion
In the following, we state the theoretical result by \citet{khemakhem:21:autoregressive} on Gaussian LSNMs. Note that we slightly changed the theorem as the original version has a typo in the definition of $g$ and $k$.

\begin{theorem}[\citet{khemakhem:21:autoregressive}]
\label{th:identifiability-autoregressive-gaussian}
Assume the data follows the model in Def.~\ref{def:lsnm} with $N_Y$ standard Gaussian, $N_Y \sim \mathcal{N}(0,1)$. If a backward model exists, i.e.
\begin{equation}
    X = h(Y) + k(Y)N_X
\end{equation}
where $N_X \sim \mathcal{N}(0,1)$, $N_X \Indep Y$ and $k > 0$, then one of the following scenarios must hold:
\begin{enumerate}
    \item $(g,f) = \left( \frac{1}{\sqrt{Q}}, \frac{P}{Q} \right)$ and $(k,h) = \left( \frac{1}{\sqrt{Q'}}, \frac{P'}{Q'} \right)$ where $Q,Q'$ are polynomials of degree two, $Q,Q' > 0$, $P,P'$ are polynomials of degree two or less, and $p_X,p_Y$ are strictly log-mix-rational-log. In particular, $\lim_{-\infty} g = \lim_{+\infty} g = 0^+$, $\lim_{-\infty} f = \lim_{+\infty} f < \infty$, similarly so for $k,h$, and $f,g,h,k$ are not invertible.
    \item $g, k$ are constant, $f,g$ are linear and $p_X,p_Y$ are Gaussian densities.
\end{enumerate}
\end{theorem}
\fi

\section{Societal Impact}
Causal inference in general has the potential to greatly improve the reliability of learning systems by uncovering causal relationships instead of, potentially spurious, correlations. Our work contributes to this endeavour and we do not anticipate any negative societal impacts because our work is mostly of theoretical nature.

\newpage
\section{Experimental Protocol and Results}
\label{app:experiments}

All considered datasets were standardized to zero mean and unit variance for our methods, and use the recommended pre-processing for the baselines. 
We use the scores of the estimators provided and only if they are strictly greater (or lower for some methods) in the right direction, a given pair is counted as correct.
For the baselines we use the standard settings of the respective papers introducing them.
Except for standardization, the datasets are not adjusted except for the T\"ubingen dataset, where we remove the $6$ multivariate and $3$ discrete pairs.
For the overall performance in Fig.~\ref{fig:overall_perf}, we weight each dataset the same irrespective of the number of pairs.

For our estimators \ourmethodmle and \ourmethodhsic, we use spline feature maps~\citep{eilers1996flexible} of order $5$ and with $25$ knots implemented in \texttt{scikit-learn}~\citep{pedregosa2011scikit}.
We run our maximum likelihood estimator \ourmethodmle for up to $100$ steps or until the likelihood change in a step is below $10^{-6}$.
For the independence-test-based estimator, we compute the residuals $r$ of the estimator and standardize them using the estimate of the standard deviation $s$ with $r = \tfrac{y - \hat{\mu}}{\hat{\sigma}}$ and then test for independence of $r$ and the corresponding input $x$.
We predict based on the $p$-value of the test in both directions.

For the neural network based estimators, NN-\ourmethodmle and NN-\ourmethodhsic, we use a neural network with a single hidden layer of width $100$ and \texttt{TanH} activation function.
The first output of the network is unrestricted and the second output applies an exponential function to ensure a positive output as described in Eq.~\ref{eq:nn_output} and App.~\ref{app:het_reg_nn}. 
The parameters are optimized by full-batch gradient-descent on the log-likelihood using \texttt{Adam}~\citep{kingma2014adam} with initial learning rate $10^{-2}$ that is decayed to $10^{-6}$ using a cosine learning rate schedule for $5 \, 000$ steps.

All numbers displayed in the main text figures are organized below in Tables~\ref{app:tab:accuracy} and \ref{app:tab:audrc} with additional baselines IGCI~\citep{janzing:12:igci} and CGNN~\citep{goudet:18:learning}.

\begin{table}[h!]
\begin{center}
\begin{tabular}{l|rrrrr|rrrrr|rrr}
\toprule
{} &            AN &           ANs &            LS &           LSs &           MNU &          SIM &         SIMc &        SIMln &         SIMG &          Tue &          Cha &          Net &        Multi \\
\midrule
NN-$\textsc{Loci}_\textrm{M}$ &  \textbf{100} &  \textbf{100} &  \textbf{100} &  \textbf{100} &  \textbf{100} &           48 &           50 &           79 &           78 &           57 &           43 &           76 &           72 \\
NN-$\textsc{Loci}_\textrm{H}$ &  \textbf{100} &  \textbf{100} &            95 &            89 &  \textbf{100} &  \textbf{79} &  \textbf{83} &           72 &           78 &           60 &  \textbf{72} &  \textbf{87} &           78 \\
$\textsc{Loci}_\textrm{M}$    &            99 &            98 &            94 &            94 &            93 &           52 &           48 &           77 &           74 &           52 &           46 &           75 &           66 \\
$\textsc{Loci}_\textrm{H}$    &            99 &            98 &            85 &            53 &            90 &           75 &           76 &           73 &  \textbf{81} &           56 &           70 &           84 &           66 \\
GRCI                          &  \textbf{100} &            94 &            98 &            87 &            88 &           77 &           77 &           77 &           70 &  \textbf{82} &           70 &           85 &           77 \\
QCCD                          &  \textbf{100} &            82 &  \textbf{100} &            96 &            99 &           62 &           72 &           80 &           64 &           77 &           54 &           80 &           51 \\
HECI                          &            98 &            55 &            92 &            55 &            33 &           49 &           55 &           65 &           56 &           71 &           57 &           72 &           91 \\
CGNN                          &            99 &            90 &            98 &            90 &            97 &           48 &           57 &           71 &           80 &           66 &           61 &           69 &           71 \\
IGCI                          &            20 &            35 &            46 &            34 &            11 &           37 &           45 &           51 &           53 &           68 &           55 &           55 &  \textbf{92} \\
CAM                           &  \textbf{100} &  \textbf{100} &  \textbf{100} &            53 &            86 &           57 &           60 &  \textbf{87} &  \textbf{81} &           58 &           47 &           78 &           35 \\
RESIT                         &  \textbf{100} &  \textbf{100} &            61 &             6 &             2 &           77 &           82 &  \textbf{87} &           78 &           57 &  \textbf{72} &           78 &           37 \\
\bottomrule
\end{tabular}
\end{center}
\vspace{-1em}
\caption{Accuracies of all methods on all benchmark data sets considered. These data were used for the barplots in the main text.}
\label{app:tab:accuracy}
\end{table}

\begin{table}[h!] 
\begin{center}
\begin{tabular}{l|rrrrr|rrrrr|rrr}
\toprule
{} &            AN &           ANs &            LS &           LSs &           MNU &          SIM &         SIMc &        SIMln &         SIMG &          Tue &          Cha &          Net &        Multi \\
\midrule
NN-$\textsc{Loci}_\textrm{M}$ &  \textbf{100} &  \textbf{100} &  \textbf{100} &  \textbf{100} &  \textbf{100} &           60 &           63 &  \textbf{95} &           89 &           66 &           47 &           86 &           93 \\
NN-$\textsc{Loci}_\textrm{H}$ &  \textbf{100} &  \textbf{100} &            99 &            97 &  \textbf{100} &           89 &  \textbf{93} &           86 &  \textbf{93} &           56 &           71 &  \textbf{97} &           77 \\
$\textsc{Loci}_\textrm{M}$    &            98 &            95 &            88 &            86 &            90 &           68 &           56 &           90 &           87 &           45 &           55 &           84 &           75 \\
$\textsc{Loci}_\textrm{H}$    &  \textbf{100} &            96 &            92 &            54 &            92 &           84 &           85 &           88 &           91 &           47 &           68 &           93 &           51 \\
GRCI                          &  \textbf{100} &  \textbf{100} &  \textbf{100} &            95 &            97 &  \textbf{90} &           92 &           92 &           88 &           73 &           71 &           96 &           74 \\
QCCD                          &  \textbf{100} &            91 &  \textbf{100} &  \textbf{100} &  \textbf{100} &           71 &           83 &           92 &           76 &  \textbf{84} &           61 &           94 &           63 \\
HECI                          &  \textbf{100} &            63 &            99 &            72 &            20 &           59 &           64 &           86 &           73 &           78 &           57 &           84 &  \textbf{99} \\
CGNN                          &  \textbf{100} &            99 &  \textbf{100} &            98 &  \textbf{100} &           59 &           66 &           89 &  \textbf{93} &           77 &           63 &           81 &           91 \\
IGCI                          &            18 &            35 &            60 &            49 &             1 &           34 &           41 &           51 &           63 &           74 &           58 &           61 &  \textbf{99} \\
CAM                           &  \textbf{100} &  \textbf{100} &  \textbf{100} &            36 &            76 &           68 &           69 &           87 &           88 &           70 &           41 &           85 &           40 \\
RESIT                         &  \textbf{100} &  \textbf{100} &            69 &             4 &             0 &           75 &           84 &           83 &           71 &           71 &  \textbf{83} &           81 &           68 \\
\bottomrule
\end{tabular}
\end{center}
\vspace{-1em}
\caption{Area under the decision rate curve (AUDRC) of all the estimators used on all benchmark data sets.
This data was used for the overall performance displayed in Fig.~\ref{fig:overall_perf}.}
\label{app:tab:audrc}
\end{table}

\clearpage
\section{Performance Comparison of LSNM Estimators}
\label{app:estimator_benchmark}

Our LSNM estimator that can achieve consistency under the right assumptions is different to the standard method commonly used for such estimation. 
Because we use the natural exponential family parametrization of the Gaussian log-likelihood, we have concavity in the linear model parameters while different common formulations do not have this property.
The de-facto standard is iterative feasible generalized least-squares (IFGLS) estimator that regresses to the mean and the squared residuals in an alternating fashion~\citep{harvey1976estimating, cawley2004heteroscedastic}.
Such estimators are mostly developed and used in the context of econometrics~\citep{amemiya1985advanced, wooldridge2015introductory} and, to the best of our knowledge, an approach based on natural parameters of the Gaussian location-scale noise model has only been proposed in the context of Gaussian process regression~\citep{le:05:heteroscedastic}.
In the following, we empirically compare our method to IFGLS.

In Fig.~\ref{app:fig:estimator_sqrt} and \ref{app:fig:estimator_lin}, we compare both estimators on a simple sinusoidal example with $x \sim \mathcal{U}[-4\pi, 4\pi]$ and $y\sim \mathcal{N}(\sin (x), 0.1 (4\pi - |x|) + 0.2)$ so that the observation noise is small at the borders and large in the middle of the problem domain.
We increase the number of samples $T$ from $100$ to $10000$ and report the KL-divergence from the estimated predictive density $q_\textrm{est}$ to the true density $p$ averaged over a grid of $10^4$ points from the left to the right boundary of the dataset ($-4\pi$ to $4\pi$).
We repeat this procedure for $100$ times and, due to numerical outliers, show the median performance of both estimators.
In Fig.~\ref{app:fig:estimator_sqrt}, we use the common heuristic of setting the number of knots of the spline features to $\sqrt{T}$ and in \ref{app:fig:estimator_lin}, we set it to $\tfrac{T}{10}$ leading to overly complex features that make fitting harder.

Overall, we observe that the proposed estimator is more robust for smaller datasets for both strategies of selecting the number of features.
Further, IFGLS fails catastrophically for too many complex features due to overfitting of the mean and thus iteratively also the variance.
Our estimator does not seem to suffer from this problem.
We hypothesize that this is due to the jointly concave objective proposed.
However, it is apparent from Fig.~\ref{app:fig:estimator_sqrt} that both estimators behave asymptotically the same.
This shows overall that our proposed estimator performs appropriately and, as can be seen on the right hand side of both figures, can successfully fit mean and variance simultaneously.
We hypothesize that both estimators work equally well given optimal feature maps and leave a thorough investigation of the proposed estimator for future work.

\begin{figure*}[h!]
\centering
\includegraphics{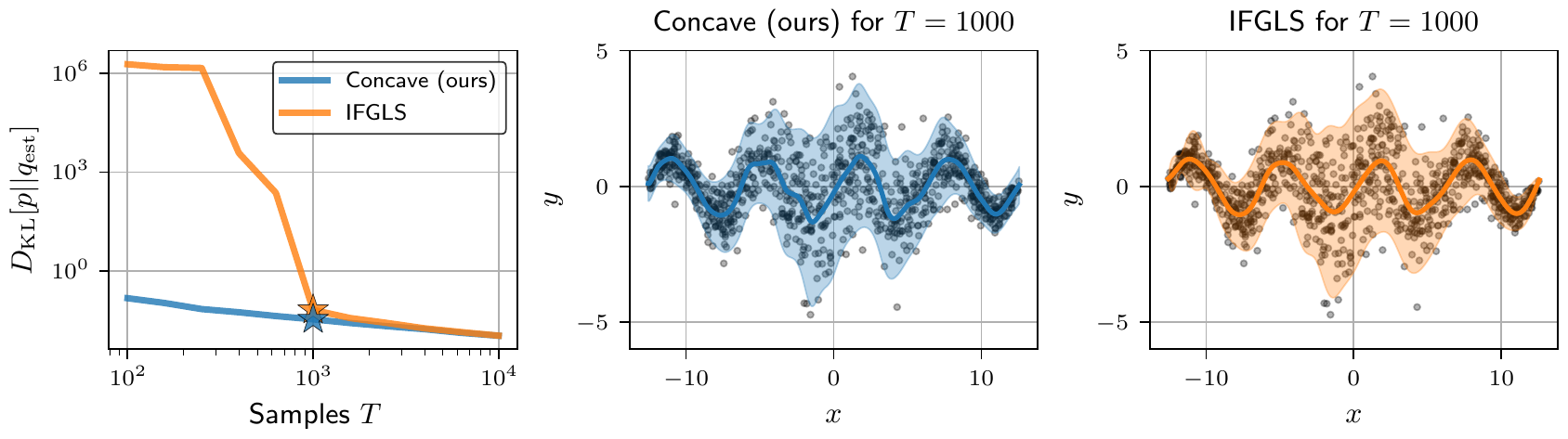}
\vspace{-2.75em}
\caption{Performance of our maximum likelihood estimator using spline features of order $5$ and the common heuristic of $\textrm{knots}=\sqrt{T}$. 
In comparison to iterative feasible generalized least-squares (IFGLS
in 
\protect\tikz[baseline=-0.15ex,inner sep=0pt]{\protect\draw[line width=6pt, C1light] (0,0.05) -- ++(0.35,0)})
, our concave estimator 
(\protect\tikz[baseline=-0.15ex,inner sep=0pt]{\protect\draw[line width=6pt, C2light] (0,0.05) -- ++(0.35,0)})
performs significantly better with fewer samples although both behave asymptotically the same.
The right two figures show the fits at $T=1000$.
}
\vspace{-1em}
\label{app:fig:estimator_sqrt}
\end{figure*}

\begin{figure*}[h!]
\centering
\includegraphics{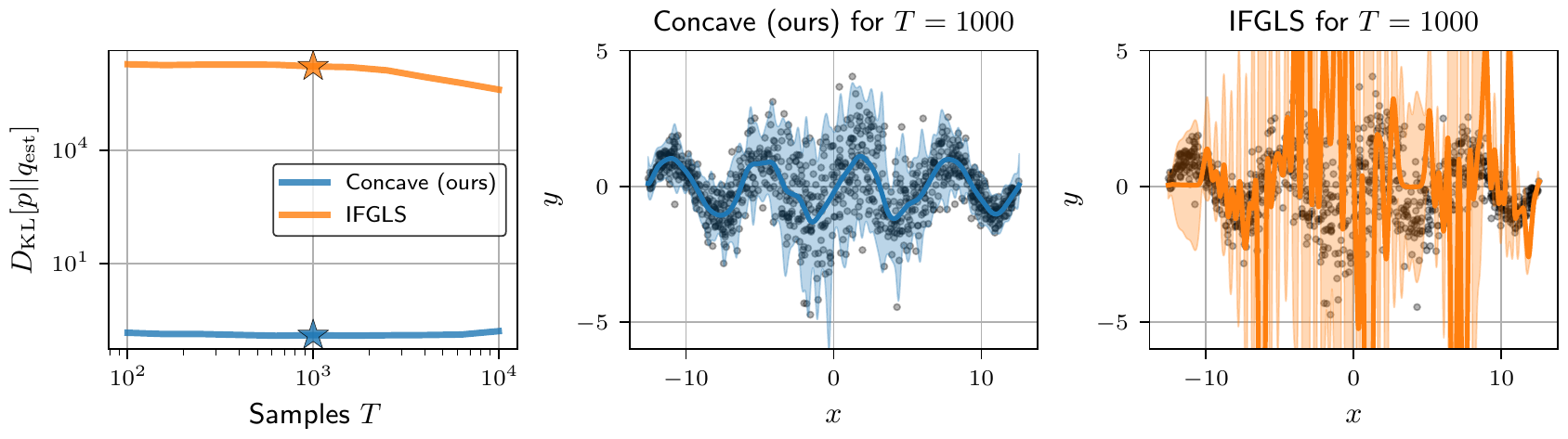}
\vspace{-2.75em}
\caption{Performance of our maximum likelihood estimator using spline features of order $5$ and more features than common with $\textrm{knots}=\tfrac{T}{10}$ with $T$ samples.
Additional complexity makes IFGLS greatly overfit and prevents it from converging while the concave estimator still works somewhat reliably.
On the right, the two fits at $T=1000$ of both methods are shown.
}
\label{app:fig:estimator_lin}
\end{figure*}

\end{document}